\documentclass[11pt, oneside]{article}   	
\usepackage[english]{babel}
\usepackage[toc,page]{appendix}
\usepackage[margin=0.75in]{geometry} 
\usepackage{hyperref}
\usepackage{graphicx,wrapfig,lipsum}	
\usepackage[labelformat=empty]{caption}
\usepackage{yfonts}	
\usepackage{amssymb}
\usepackage{amsmath}
\usepackage{amsthm}
\usepackage{scalerel}
\usepackage{verbatim}
\usepackage{bbm}
\usepackage[lined,boxed,ruled,norelsize,algo2e,linesnumbered]{algorithm2e}

\usepackage{tikz}
\usepackage{enumitem}
\usetikzlibrary{calc}
\usetikzlibrary{patterns}

\tikzstyle{every node}=[circle, draw, fill=black!50, inner sep=0pt, minimum width=4pt]
\tikzstyle{rouge}=[circle, draw, fill=red, inner sep=0pt, minimum width=6pt]
\tikzstyle{bleu}=[circle, draw, fill=blue, inner sep=0pt, minimum width=6pt]
\tikzstyle{petitrouge}=[circle, draw, fill=red, inner sep=0pt, minimum width=4pt]
\tikzstyle{petitbleu}=[circle, draw, fill=blue, inner sep=0pt, minimum width=4pt]
\tikzstyle{txt}=[draw=none, fill=none]
\tikzstyle{white}=[circle, draw, fill=black!0, inner sep=0pt, minimum width=4pt]

\definecolor{ffqqqq}{rgb}{1,0,0}
\definecolor{qqzzqq}{rgb}{0,0.6,0}
\definecolor{qqqqff}{rgb}{0,0,1}
\definecolor{uuuuuu}{rgb}{0.27,0.27,0.27}

\newcommand{\p}{\mathbf{p}}
\newcommand{\q}{\mathbf{q}}
\newcommand{\x}{\mathbf{x}}
\newcommand{\y}{\mathbf{y}}

\def\L{\mathcal{L}}

\def\Feas{\mathrm{Feas}}
\def\range{\mathrm{range}}
\def\gap{\mathrm{gap}}
\def\CHILDREN{\mathrm{CHILDREN}}
\def\PARENT{\mathrm{PARENT}}
\def\ROOT{\mathrm{ROOT}}

\usepackage[colorinlistoftodos]{todonotes}

\newtheorem{thm}{Theorem}[section]
\newtheorem{lem}[thm]{Lemma}
\newtheorem{prop}[thm]{Proposition}

\theoremstyle{definition}
\newtheorem{defn}{Definition}

\title{Cooperative and Stochastic Multi-Player Multi-Armed Bandit:\\
Optimal Regret With Neither Communication Nor Collisions}
\author{S\'ebastien Bubeck \\
Microsoft Research
\and Thomas Budzinski \\
UBC
\and 
Mark Sellke \\
Stanford University}

\definecolor{darkgreen}{rgb}{0.0, 0.5, 0.0}

\begin{document}

\maketitle

\abstract{We consider the cooperative multi-player version of the stochastic multi-armed bandit problem. We study the regime where the players cannot communicate but have access to shared randomness. In prior work by the first two authors, a strategy for this regime was constructed for two players and three arms, with regret $\tilde{O}(\sqrt{T})$, and with no collisions at all between the players (with very high probability). In this paper we show that these properties (near-optimal regret and no collisions at all) are achievable for any number of players and arms. At a high level, the previous strategy heavily relied on a $2$-dimensional geometric intuition that was difficult to generalize in higher dimensions, while here we take a more combinatorial route to build the new strategy.}

\section{Introduction} \label{sec:intro}
We consider the cooperative multi-player version of the classical stochastic multi-armed bandit problem. We denote by $m$ the number of players and by $K\geq m$ the number of arms. The bandit instance is described by the mean rewards $\p = \left( p(1), \hdots, p(K) \right) \in [0,1]^K$, which is unknown to the players. Denote $(Y_t(i))_{1 \leq i \leq K, 1 \leq t \leq T}$ for a sequence of independent random variables such that $\mathbb P(Y_t(i) = 1) = p(i)$ and $\mathbb P(Y_t(i) = 0) = 1-p(i)$. At each time step $t=1, \hdots, T$, each player $X \in [m]$ chooses an action $i_t^X \in [K]$, and observes the corresponding reward $Y_t(i_t^X)$. We define the regret by:
\[
R_T = T \cdot \max_{\mathbf{a} \in \{0,1\}^K : \sum_{i=1}^K a(i) = m} \mathbf{a} \cdot \p - \sum_{t=1}^T \sum_{X = 1}^m p(i_t^X).
\]
We assume that once the game has started the players cannot communicate at all (but they can agree on a strategy prior to the game starting).
Usually in these cooperative multi-player bandit problems the players are also penalized for {\em collisions}, i.e. time steps $t$ such that $i_t^X = i_t^Y$ for some distinct players $X \neq Y$. Here instead, following the prior work \cite{BB20}, we consider a seemingly daunting constraint on the players that subsumes such penalty: we ask that with high probability (with respect to the reward generation process) they simply {\em do not collide at all}. Our main result reads as follows:
\begin{thm} \label{thm:bandit}
There exists a randomized strategy for the players based on shared randomness\footnote{The shared randomness assumption only requires the players to have access to a public source of randomness (e.g., unpredictable weather patterns). Indeed, the ``adversary" that selects the bandit instance $\p$ can have access to this source too, but the point is that once $\p$ is chosen it is fixed for the rest of the game, and so the choice cannot depend on future outcomes of the randomness source.}, such that for any $\p \in [0,1]^K$, we have the regret bound
\[
\mathbb E[R_T] \leq O \left( m K^{11/2} \sqrt{T \log(T)} \right) ,
\]
and furthermore with probability at least $1-\frac{1}{T}$ (with respect to the reward generation process) the players never collide.
\end{thm}

\paragraph{Related works.} The mutiplayer bandit problem with limited communication was first introduced roughly at the same time in \cite{LJP08, LZ10, AMTS11}, and has been extensively studied since then \cite{AM14, RSS16, BBMKP17, LM18, BP18, ALK19, BLPS20}, with various assumptions on the communication/collisions. Yet it is only in \cite{BB20} that it was realized that one could in fact obtain the optimal regret {\em without any collisions at all}. The latter result was however limited to $2$ players and $3$ actions. Indeed the strategy was based on a simple $2$-dimensional construction, reproduced here in Figure \ref{fig_partitionBB}. The extension even to $4$ actions seemed very difficult. In the present paper we take a more combinatorial route, which lends itself better to the high-dimensional picture. The analogue of Figure \ref{fig_partitionBB} for our new strategy is Figure \ref{fig:2of3actions} (see below, in the high-level overview of our strategy, for more details on the meaning of the partitions depicted in these figures).







\begin{figure}
\centering
\vspace{-0.5cm}
  \includegraphics[width=0.5\linewidth]{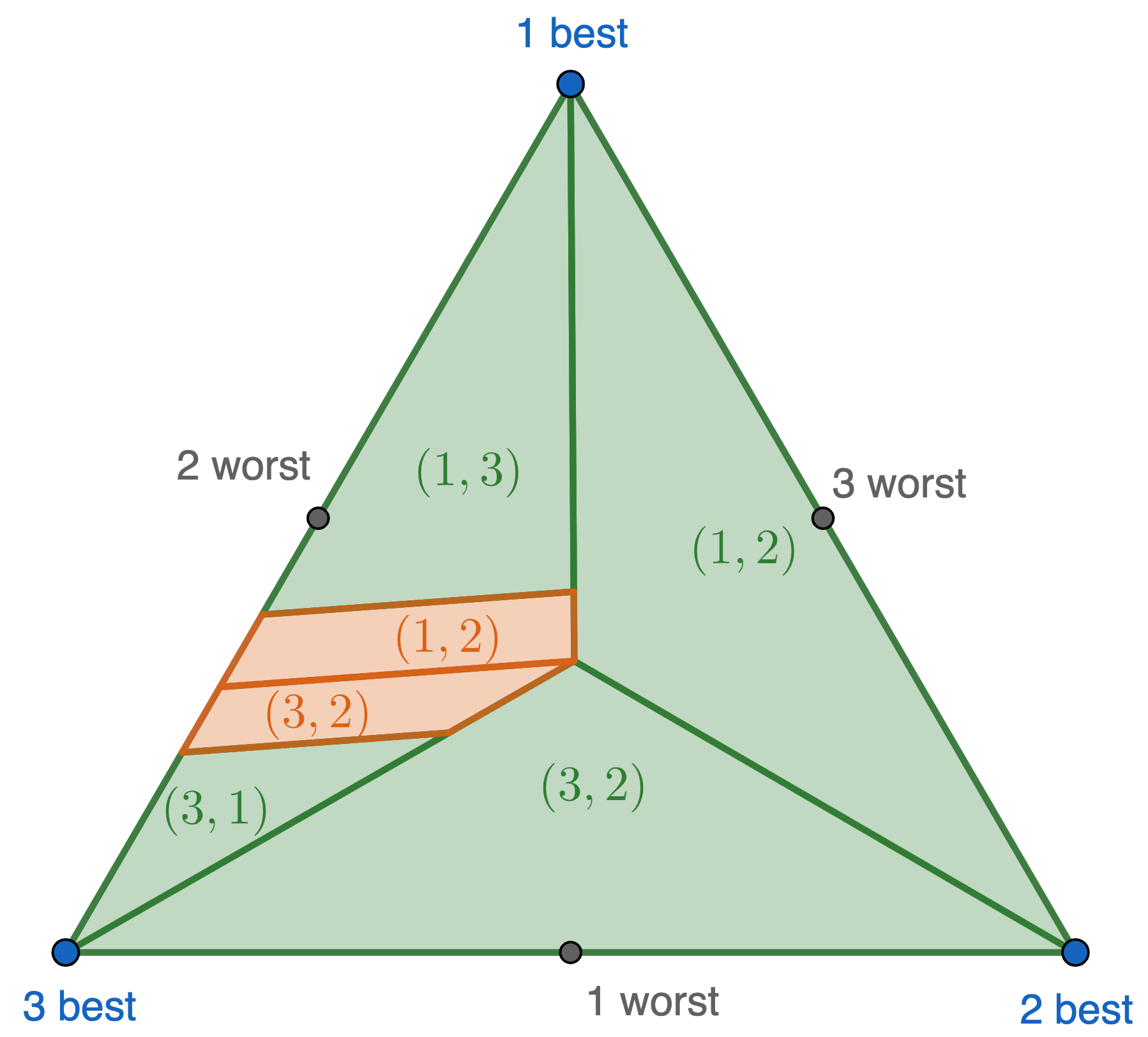}
  \caption{The basic partition in \cite{BB20}, restricted to a plane where $p(1)+p(2)+p(3)$ is constant. The parts are labelled by the actions each player plays based on their estimate of $\p$. Crucially, adjacent regions never result in collisions. Except for the small orange region, the players always play the top $2$ actions.}
  \includegraphics[width=\linewidth]{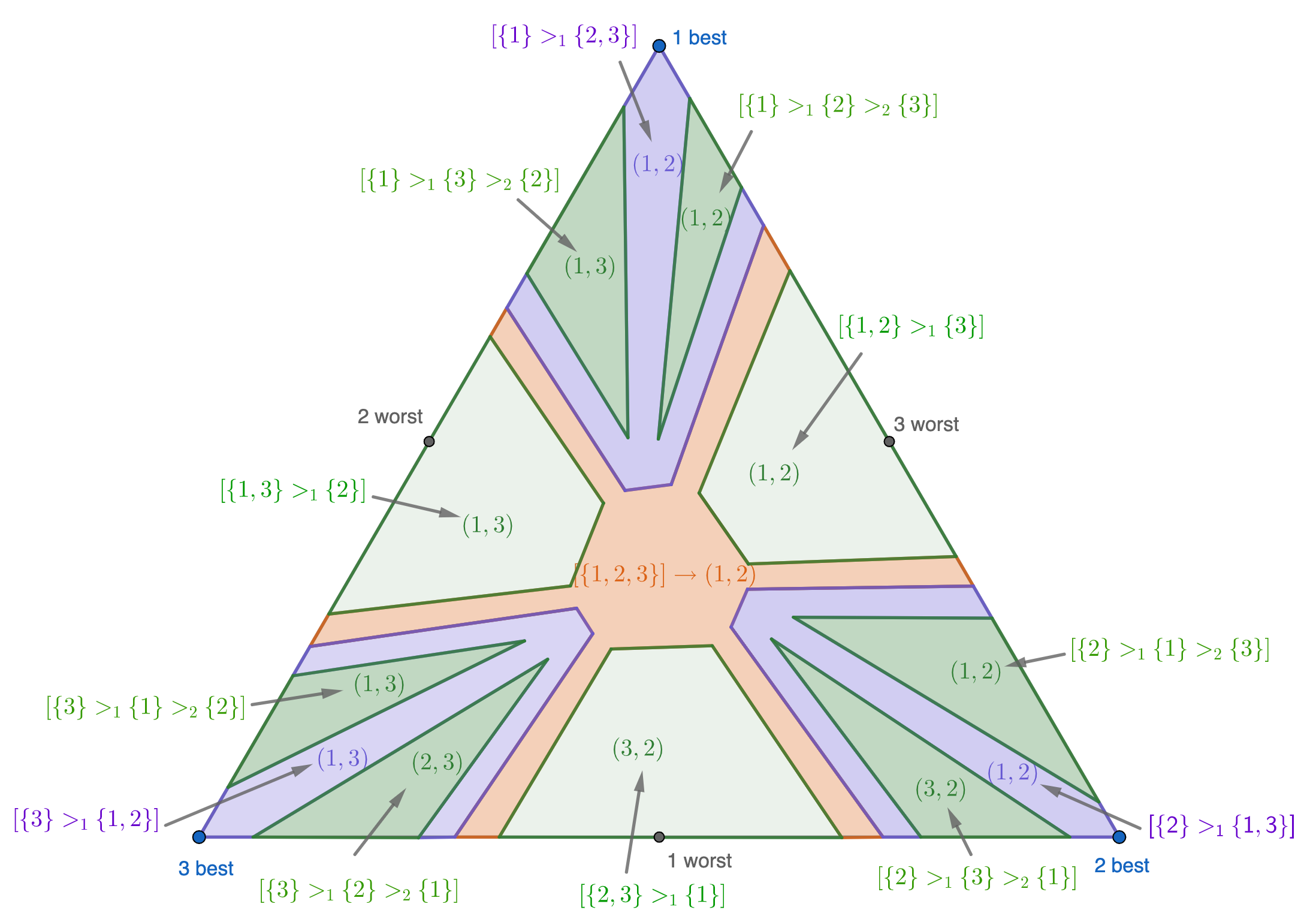}
  \caption{Our new partition for $2$ players and $3$ actions, also restricted to a plane where $p(1)+p(2)+p(3)$ is constant. The parts are labelled by pairs of actions as before, as well as vertices of the tree $\mathcal T_{K,m}$ defined in subsection~\ref{subsec:Tkm}. There are now two levels of ``skeleton" shown in orange and purple, separating the main parts in which the top $2$ actions are always played. Again adjacent regions never result in a collision. In contrast with the partition above, our new partition extends to general $(K,m)$.}\label{fig_partitionBB}
  \label{fig:2of3actions}
\end{figure}


\paragraph{Paper organization.} We start in Section \ref{sec:partition} with our main construction, a certain ``stable" partition of the cube $[0,1]^K$. Following \cite{BB20} we then put this partition to work in Section \ref{sec:fullinfo} on a simpler version of the problem, the so-called {\em full information case}, where the players observe at each round a reward on {\em all} the arms, but these rewards are independent realizations for each player. Finally, in Section \ref{sec:bandit}, we show how to deal with the extra difficulty of exploration for the bandit setting, and we complete the proof of Theorem \ref{thm:bandit}.

\paragraph{High level overview of our strategy.} Our strategy is based on a (time-dependent) colored partition of $[0,1]^K$, where the colors correspond to $m$-tuples of arms. At a given time step, players decide which arm to play based on the color of the partition element that contains their respective current empirical estimate of $\p$. We describe our partition using a decision tree, where each node is labeled with a set of currently undecided arms (i.e. for which we have not yet decided whether they belong to the top $m$ actions or not) and the children of this node correspond to the possible subdivisions of this undecided set. This naturally gives a hierarchical decomposition of $[0,1]^K$, and one of our key insights is to assign a partition element not only to the leaves of this decomposition, but also to any part which is {\em close} to a decision boundary at some inner node. The effect of the latter modification is a certain {\em stability} property of the partition: as long as the empirical estimates of the players remain close, the partition elements that decide the players' actions will be neighbors in the decision tree (Lemma \ref{lem:topology}). Thus to avoid collision it suffices to find a coloring such that two players never collide if they play the arms recommended by two neighboring nodes in the decision tree (Lemma \ref{lem:color}). What about the regret of such a strategy? If players play according to a leaf of the tree, then in fact they will act optimally (Lemma~\ref{lem:guarantee}) or close to optimally, up to the deviations between $\p$ and its estimates. The question is how to control the regret from inner nodes, which correspond to the interfaces on Figure~\ref{fig:2of3actions}, where the players make suboptimal choices. For this, one needs to have the geometric picture in mind: to stay at an inner node means that one was close to a decision boundary, which will only happen for a ``small" fraction of the instances $\p$. To turn things around, as in \cite{BB20}, we choose a {\em random} decision boundary, so that for any fixed instance $\p$, the probability to be in the boundary will be small. It remains to carefully analyze how the ``small" probability relates to the regret from playing suboptimally in this boundary (Lemma \ref{lem:cost}).

\section{A partition of the hypercube} \label{sec:partition}
In this section, which comprises the main new contribution of this work, we construct our randomized partition of the hypercube $[0,1]^K$ together with the corresponding piecewise constant strategy for the $m$ players. This was the main barrier in extending \cite{BB20} beyond the case of $2$ players and $3$ actions.

\subsection{The tree $\mathcal T_{K,m}$}

\label{subsec:Tkm}

The most basic goal of our construction is to partition $[0,1]^K$ based on the order of the coordinates. We begin by giving the combinatorial setup which underlies this construction. We first consider the class of \emph{ordered partitions} of $[K]$. These are set partitions of $[K]$ where the elements within a part are not ordered, but the parts are ordered. That is, an ordered set partition of $[K]$ has the form
\[P=\left[S_1> S_2 >\dots >S_j\right],\]
where $(S_i)_{i=1}^j$ partition $[K]$. For example $\left[\{1,3,5\}>\{2,6,7\}>\{4\}\right]$ is an ordered set partition of $\{1,2,3,4,5,6,7\}$ and is identical to $\left[\{5,1,3\}>\{6,2,7\}>\{4\}\right]$. We will construct these inequalities ``one at a time" and it turns out we need to keep track of the order in which they are introduced. Therefore we define a \emph{doubly ordered partition} (henceforth DOP) to be an ordered set partition in which the inequality signs are themselves ordered. Thus a DOP of $[K]$ has the form
\[P=\left[S_1>_{\sigma(1)} S_2 >_{\sigma(2)}\dots >_{\sigma(j-1)} S_j\right]\]
for some permutation $\sigma\in \mathfrak{S}_{j-1}.$ For example \[\left[\{1,3,5\}>_1\{2,6,7\}>_2\{4\}\right] \quad \mbox{and} \quad \left[\{1,3,5\}>_2\{2,6,7\}>_1\{4\}\right]\] are the two DOPs with underlying ordered partition $\left[\{1,3,5\}>\{2,6,7\}>\{4\}\right]$. The set of DOPs that we have just defined has a natural tree structure, and we denote this tree by $\mathcal T_K$. More precisely, the root of $\mathcal T_K$ is the trivial DOP $\mathrm{ROOT} := \left[\{1,2,\dots,K\}\right]$ and, for every DOP
\[P_1=\left[S_1>_{\sigma(1)} S_2>_{\sigma(2)}\dots >_{\sigma(i-1)}S_i >_{j-1} S_{i+1}>_{\sigma(i+1)}\dots>_{\sigma(j-1)}S_j\right]\]
that is $\ROOT$ (i.e. with $j \geq 2$), the parent of $P_1$ is
\[P=\left[S_1>_{\sigma(1)} S_2>_{\sigma(2)}\dots >_{\sigma(i-1)}S_i \cup S_{i+1}>_{\sigma(i+1)}\dots>_{\sigma(j-1)}S_j\right].\]
In other words, descending the tree $\mathcal T_k$ amounts to adding inequalities $>_1, >_2, \dots$ in this order.

The tree we will work with throughout this paper is a subtree $\mathcal T_{K,m}\subseteq\mathcal T_K$. The reason is that we are only concerned with identifying the set of the top $m$ actions, potentially without knowing the relative order within these top $m$. To focus on this, we introduce the following definitions. Let $i(P)$ be the largest integer $i \geq 0$ such that $\sum_{j=1}^i |S_j| \leq m$ (for example $i(\mathrm{ROOT})  = 0$). We define the set $A(P):=S_1 \cup \hdots \cup S_{i(P)}$ (with the convention $A(P) = \emptyset$ if $i(P)=0$), which corresponds to the set of actions that the DOP $P$ has already identified as being in the top $m$ actions. We now define the set $B(P)$ of actions that needs to partitioned further to fully identify the top $m$ actions: if $|A(P)|=m$ then $B(P) := \emptyset$, and otherwise $B(P) := S_{i(P)+1}$. We can now define $\mathcal T_{K,m}\subseteq \mathcal T_K$ as the subtree formed by paths from the root where only inequalities involving $B(P)$ may be added to a DOP $P$ at any time. In other words, we define $\mathcal{T}_{K,m}$ recursively as follows: let
\[P_1=\left[S_1>_{\sigma(1)} S_2>_{\sigma(2)}\dots >_{\sigma(i-1)}S_i >_{j-1} S_{i+1}>_{\sigma(i+1)}\dots>_{\sigma(j-1)}S_j\right]\]
be a DOP and let
\[P=\left[S_1>_{\sigma(1)} S_2>_{\sigma(2)}\dots >_{\sigma(i-1)}S_i \cup S_{i+1}>_{\sigma(i+1)}\dots>_{\sigma(j-1)}S_j\right]\]
be its parent. If $P \in \mathcal{T}_{K,m}$, then $P_1 \in \mathcal{T}_{K,m}$ if and only if $B(P)=S_i \cup S_{i+1}$. See Figure~\ref{fig:tree23} for an example. We also denote by $\L(\mathcal T_{K,m})$ the set of leaves of the tree $\mathcal{T}_{K,m}$. Note that the leaves of $\mathcal T_{K,m}$ are DOPs which determine the top $m$ actions. However not all DOPs determining the top $m$ actions are leaves of $\mathcal T_{K,m}$. For example we have \[[\left\{4,8\}>_2\{2,6,7\}>_1\{1,3,5\}\right]\in\mathcal T_{8,2} \quad \text{but} \quad \left[\{4,8\}>_1\{2,6,7\}>_2\{1,3,5\}\right]\notin\mathcal T_{8,2}.\] The latter holds because the parent DOP $\left[\{4,8\}>_1\{1,3,5,2,6,7\}\right]$ determines the top $2$ actions hence is already a leaf of $\mathcal T_{8,2}$.

\begin{figure}
\begin{center}
\begin{tikzpicture}
\draw(0.9,0)--(2.2,1.5);
\draw(0.9,0)--(2.2,0.9);
\draw(0.9,0)--(2.2,0.3);
\draw(0.9,0)--(2.2,-0.3);
\draw(0.9,0)--(2.2,-0.9);
\draw(0.9,0)--(2.2,-1.5);

\draw(4.8,0.9)--(6.2,1.2);
\draw(4.8,0.9)--(6.2,0.6);
\draw(4.8,-0.3)--(6.2,0);
\draw(4.8,-0.3)--(6.2,-0.6);
\draw(4.8,-1.5)--(6.2,-1.2);
\draw(4.8,-1.5)--(6.2,-1.8);

\draw(0,0)node[txt]{$[\{1,2,3\}]$};

\draw(3.5,1.5)node[txt]{$[\{1,2\} >_1 \{3\}]$};
\draw(3.5,0.9)node[txt]{$[\{1\} >_1 \{2,3\}]$};
\draw(3.5,0.3)node[txt]{$[\{1,3\} >_1 \{2\}]$};
\draw(3.5,-0.3)node[txt]{$[\{3\} >_1 \{1,2\}]$};
\draw(3.5,-0.9)node[txt]{$[\{2,3\} >_1 \{1\}]$};
\draw(3.5,-1.5)node[txt]{$[\{2\} >_1 \{1,3\}]$};

\draw(8,1.2)node[txt]{$[\{1\} >_1 \{2\} >_2 \{3\}]$};
\draw(8,0.6)node[txt]{$[\{1\} >_1 \{3\} >_2 \{2\}]$};
\draw(8,0)node[txt]{$[\{3\} >_1 \{1\} >_2 \{2\}]$};
\draw(8,-0.6)node[txt]{$[\{3\} >_1 \{2\} >_2 \{1\}]$};
\draw(8,-1.2)node[txt]{$[\{2\} >_1 \{1\} >_2 \{3\}]$};
\draw(8,-1.8)node[txt]{$[\{2\} >_1 \{3\} >_2 \{1\}]$};
\end{tikzpicture}
\end{center}
\vspace{-1cm}
\caption{The tree $\mathcal{T}_{3,2}$, with $9$ leaves and $4$ inner nodes.}\label{fig:tree23}
\end{figure}
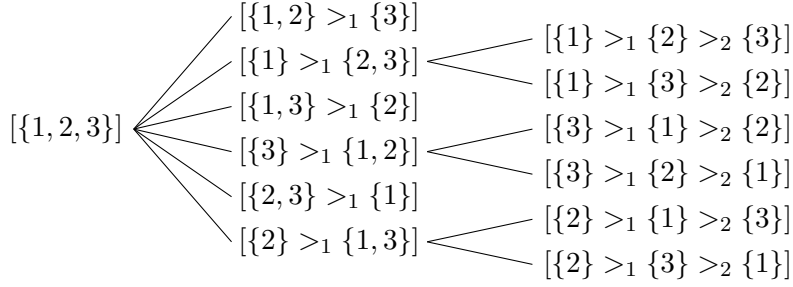

For $P\in\mathcal T_{K,m}$ we use $\mathrm{CHILDREN}(P)$ to denote its set of children, and $\mathrm{PARENT}(P)$ to denote its (unique) parent. For example we have
\[\mathrm{PARENT}\left(\left[\{1,3,5\}>_1\{2,6,7\}>_2\{4\}\right]\right)~=~\left[\{1,3,5\}>_1\{2,4,6,7\}\right].\] For convenience we will treat $\mathcal T_{K,m}$ as a partial order, so that $Q\preceq P$ means $Q$ is a ancestor of $P$. In particular, the root satisifes $\mathrm{ROOT}\preceq P$ for any $P\in\mathcal T_{K,m}$. Finally we denote by $d_{\mathcal T_{K,m}}$ the graph distance in the tree $\mathcal T_{K,m}$.

\subsection{Inequalities to define a region}
Our goal is now to connect DOPs to regions inside $[0,1]^K$. To do so we will use certain types of inequalities involving the values of the coordinates of $\x=\left( x(i) \right) \in [0,1]^K$. We first need a few definitions.

\begin{defn}
Let $\x \in [0,1]^K$ and $P\in\mathcal T_{K,m} \setminus \L(\mathcal T_{K,m})$. We define 
\[\range_{P}(\x):=\max_{k\in B(P)} x(k)- \min_{\ell \in B(P)} x(\ell).\] 
\end{defn}

\begin{defn}
Let $\x \in [0,1]^K$ and $P$ be a DOP of the form:
\begin{equation} \label{eq:DOPform}
P=\left[S_1>_{\sigma(1)} S_2>_{\sigma(2)}\dots >_{\sigma(i-1)}S_i >_{j-1} S_{i+1}>_{\sigma(i+1)}\dots>_{\sigma(j-1)}S_j\right].
\end{equation}
We define \[\gap_{P}(\x)=\min_{k \in S_{i}}x(k)-\max_{\ell \in S_{i+1}} x(\ell).\]
\end{defn}

In words, $\range_P(\x)$ represents the range of values in the set of coordinates for which the DOP $P$ has not yet identified whether they are in the top $m$ actions or not. On the other hand, $\gap_P(\x)$ represents how large was the ``cut" made by the DOP $P$ when we added its last inequality. The next easy lemma says that there always exists a ``large cut".

\begin{lem}
\label{prop:cover}
Let $\x \in [0,1]^K$ and $P\in\mathcal T_{K,m} \setminus \L(\mathcal T_{K,m})$. There exists a DOP $Q \in \mathrm{CHILDREN}(P)$ such that
\[\gap_{Q}(\x)\geq \frac1{K}\cdot \range_{P}(\x)\geq 0.\]
\end{lem}

\begin{proof}
Write $B(P)=\{a_1,\dots,a_{\ell}\}$, with $x(a_1)\geq x(a_2)\geq\dots\geq x(a_{\ell})$ (note that $\ell \geq 2$ since $P$ is not a leaf). The pigeonhole principle implies that some adjacent pair of values $x(a_j),x(a_{j+1})$ will differ by at least $\frac{\range_P(\x)}{K}$, so one can simply set $Q$ to be the child of $P$ which separates $B(P)$ into $\{ x(a_1), \dots, x(a_j) \}$ and $\{ x(a_{j+1}), \dots, x(a_{\ell}) \}$.
\end{proof}

\subsection{Constructing the partition}

\begin{figure}[t]

\SetKwFor{Loop}{loop}{}{end}

\begin{algorithm2e}[H]\label{alg:partition}
\caption{Definition of the mapping $\mathcal P_{c, \varepsilon} : [0,1]^K \rightarrow \mathcal T_{K,m}$.}

\SetAlgoLined\DontPrintSemicolon

\textbf{parameters:} $c:\mathcal T_{K,m}\to [0,\frac{1}{K}]$, $\varepsilon>0$, \textbf{input:} $x\in [0,1]^K$, \textbf{returns:} vertex $P \in \mathcal T_{K,m}$.

Initialize $P= \mathrm{ROOT}$.

\While{$P\notin \L(\mathcal T_{K,m})$}{\label{while}

\For{$Q\preceq P$}{

Write $B(Q)=\{a_1,\dots,a_{\ell}\}$, with $x(a_1)\geq x(a_2)\geq\dots\geq x(a_{\ell})$. \label{writeB(Q)} \\
\For{$j=1,2,\dots,\ell-1$}{%
Define the child $Q_j$ of $Q$ by splitting $B(Q)$ into $\{a_1,\dots,a_j\}>\{a_{j+1},\dots,a_{\ell}\}.$

\uIf{$\left|\gap_{Q_j}(\x)-c(Q)\cdot \range_Q(\x)\right| \leq (d_{\mathcal T_{K,m}}(P,Q) +1) \cdot 6\varepsilon$ \label{line:newskel}}{\Return $P$}

}

}

Write $B(P)=\{a_1,\dots,a_{\ell}\}$, with $x(a_1)\geq x(a_2)\geq\dots\geq x(a_{\ell})$. \\
\For{$j=1,2,\dots,\ell-1$}{%
Define the child $P_j$ of $P$ by splitting $B(P)$ into $\{a_1,\dots,a_j\}>\{a_{j+1},\dots,a_{\ell}\}.$

\uIf{$\gap_{P_j}(\x) \geq c(P) \cdot \range_P(\x)$ \label{line:childineq}}{

\tcp*[f]{By Lemma \ref{prop:cover} and $c(P) \leq \frac{1}{K}$, this occurs for at least one $j \in [\ell-1]$.}

$P\leftarrow P_j$\label{line:child} \\

\textbf{break} (go back to line~\ref{while})}

}

\tcp*[f]{While loop terminated so $P$ is a leaf.
}}

\Return $P$

\end{algorithm2e}
\end{figure} 

We now finally construct the partition of $[0,1]^K$, depending in a deterministic way on a function $c:\mathcal T_{K,m}\to [0,\frac{1}{K}]$ as well as a small constant $\varepsilon>0$. The partition elements will be indexed by vertices of the tree $\mathcal{T}_{K,m}$, or in other words the partition is defined by a mapping $\mathcal P_{c, \varepsilon} : [0,1]^K \rightarrow \mathcal T_{K,m}$. This mapping is easiest to describe algorithmically, which we do in Algorithm \ref{alg:partition}.

Let us now comment on what Algorithm~\ref{alg:partition} is doing (the reader might also find useful to look at Figure~\ref{fig:2of3actions} at the same time). First, lines~\ref{line:childineq} and~\ref{line:child} are the recursive step, where we decide which is the next inequality that we add to our DOP. This decision depends on the parameters $c(P)$. Moreover, line~\ref{line:newskel} for $Q=P$ means that if one of these decisions is $6 \varepsilon$-close to make, we simply output $P$. This corresponds to the interfaces of width $6\varepsilon$ between large cells on Figure~\ref{fig:2of3actions}. Note that the positions of these interfaces depend on the $c(P)$. Moreover, for a given $Q$, the condition of line~\ref{line:childineq} becomes weaker as the algorithm progresses and $P$ gets deeper in the tree $\mathcal{T}_{K,m}$. Therefore, a point that was very close to being assigned to $Q$ will be assigned to its child. This results in ``coating" of the interfaces with layers of width $6\varepsilon$. For example, on Figure~\ref{fig:2of3actions}, this corresponds to the region $[\{1\}>_1\{2,3\}]$ separating the regions $[\{1,2,3\}]$ and $[\{1\}>_1\{2\}>_2\{3\}]$. In general, there can be up to $K$ such interface layers.

A crucial property of the partition is the following stability property. The first item will be useful to ensure the absence of collisions. The goal of the second is to state a ``consistency" property for different values of $\varepsilon$, which will be needed later in the bandit analysis.

\begin{lem}
\label{lem:topology}
We fix $\varepsilon>0$, $c : \mathcal{T}_{K,m} \to \left[ 0, \frac{1}{K} \right]$, and $\x, \y \in [0,1]^K$.
\begin{enumerate}
\item
If $|\x-\y|_{\ell^{\infty}}\leq\varepsilon$, then $d_{\mathcal T_{K,m}} \left( \mathcal P_{c, \varepsilon}(\x), \mathcal P_{c, \varepsilon}(\y) \right) \leq 1$.
\item
Let $P \in \mathcal{T}_{K,m}$ and assume that $|x(i)-y(i)| \leq \varepsilon$ for all $i \in A(P) \cup B(P)$. Let also $\varepsilon' \in (0,\varepsilon]$. Then it is not possible that $\mathcal P_{c, \varepsilon}(\x)$ and $\mathcal P_{c, \varepsilon'}(\y)$ are descendants of two distinct children of $P$.
\end{enumerate}
\end{lem}

\begin{proof}
We start with the second point, since it will be useful for the first one.
Assume that we run Algorithm~\ref{alg:partition} on both $\x$ with $(c,\varepsilon)$ and on $\y$ with $(c, \varepsilon')$ in parallel and that both instances are currently both on $P$ (in the while loop). Then we need to prove that the two instances do not branch into two distinct children of $P$. If $P$ is the final output for either $\x$ or $\y$ then the conclusion is immediate. If not, we denote by $P_j(\x)$ (resp. $P_j(\y)$) the $j$-th child of $P$ considered by Algorithm~\ref{alg:partition} run on $\x$ (resp. on $\y$). Note that the coordinates may be ordered differently in $\x$ and in $\y$, so it not clear that the children of $P$ considered for $\x$ and $\y$ should be the same. Assume that an instruction $P\leftarrow P_j$ occurs for $\x$ at $j=j_{\x}$ (resp. at $j=j_{\y}$ for $\y$).
If $P\leftarrow P_{j_{\x}}(\x)$ occurs for $\x$ in line~\ref{line:child} of the algorithm, we have
\begin{equation}\label{eqn:gap_positive}
\gap_{P_{j_{\x}}(\x)}(\x) \geq c(P) \cdot \range_P(\x).
\end{equation}
On the other hand, since $P$ was not the output for $\x$, the inequality of line~\ref{line:newskel} is not satisfied, which means that the distance between the left and the right-hand side of the last display is at least $6\varepsilon$, i.e.
\begin{equation}\label{eqn:gap_4epsilon}
\gap_{P_{j_{\x}}(\x)}(\x) \geq c(P) \cdot \range_P(\x)+6\varepsilon.
\end{equation}
Hence, since $|\x-\y|_{\ell^{\infty}} \leq \varepsilon$, we also have
\[\gap_{P_{j_{\x}}(\x)}(\y) > c(P) \cdot \range_P(\y).\]
In particular, the fact that the left-hand side is positive means that $P_{j_{\x}}(\x)$ is also one of the children of $P$ considered when the algorithm is run on $\y$, i.e. $P_{j_{\x}}(\y)=P_{j_{\x}}(\x)$. Moreover, the last display means that the inequality in line~\ref{line:childineq} is satisfied when the algorithm considers $P_{j_{\x}}(\y)=P_{j_{\x}}(\x)$ for $\y$. This proves that $j_{\y} \leq j_{\x}$.

We now assume $j_{\y} < j_{\x}$ and will reach a contradiction (note that the argument is not symmetric in $\x$ and $\y$ since we assume $|\x-\y| \leq \varepsilon$ and not $|\x-\y| \leq \varepsilon'$). Similarly to Equation~\eqref{eqn:gap_4epsilon}, using the fact that $P$ is not the output for $\x$ and that $P  \leftarrow P_{j_{\y}}$ does not occur when we run the algorithm on $\x$, we have
\[\gap_{P_{j_{\y}}(\x)}(\x) \leq c(P) \cdot \range_P(\x) -6\varepsilon.
\]
Using $|\x-\y|_{\ell^{\infty}} \leq \varepsilon$, we deduce
\[\gap_{P_{j_{\y}}(\x)}(\y) \leq c(P) \cdot \range_P(\y) -4\varepsilon.
\]
On the other hand, for the exact same reason as in Equation~\eqref{eqn:gap_4epsilon}, we have
\[\gap_{P_{j_{\y}}(\y)}(\y) \geq c(P) \cdot \range_P(\y)+6\varepsilon' > c(P) \cdot \range_P(\y).
\]
From the last two displays, we obtain
\[ \left| \gap_{P_{j_{\y}}(\x)}(\y)-\gap_{P_{j_{\y}}(\y)}(\y) \right| > 4\varepsilon.\]
On the other hand, we write $B(P)=\{a_1,\dots,a_{\ell}\}=\{b_1, \dots, b_{\ell}\}$, where $x(a_1) \geq \dots \geq x(a_{\ell})$ and $y(b_1) \geq \dots \geq y(b_{\ell})$. Then by definition of $\gap$, the last display becomes
\begin{equation}\label{eqn_gap_x_y}
\left| \left( y(b_{j_{\y}})-y(b_{j_{\y}+1})\right)-\left( y(a_{j_{\y}})-y(a_{j_{\y}+1})\right) \right| > 2\varepsilon.
\end{equation}
On the other hand, using the assumption $|\x-\y|_{\ell^{\infty}}$, we have
\[ \left| y(b_{j_{\y}}) - y(a_{j_{\y}}) \right| \leq \left| y(b_{j_{\y}}) - x(a_{j_{\y}}) \right| + \left| x(a_{j_{\y}}) - y(a_{j_{\y}}) \right| \leq 2\varepsilon \]
and similarly for $j_{\y}+1$. This contradicts Equation~\eqref{eqn_gap_x_y}, so we obtain $j_{\x}=j_{\y}$. As explained in the first part of the proof, we have $P_{j_{\x}}(\x)=P_{j_{\x}}(\y)$, so both instances of the algorithm branch in this child of $P$, which proves the second item of the Lemma.

We now prove the first item. We first use the second item with $\varepsilon'=\varepsilon$. Note that the assumption $|\x-\y|_{\ell^{\infty}} \leq \varepsilon$ is stronger than the assumption needed on the second item. Therefore, we know that when we run the algorithm on $\x$ and $\y$ in parallel, the runs agree as long as both instances are still running. Now assume without loss of generality that $\x$ stops first. If $\x$ stops at a leaf $P \in \L(\mathcal T_{K,m})$ then we already know by the second item that $\y$ also stops at $P$. So let us assume that $\x$ stops at an inner node $P$, and that $\y$ continues to a child $P'$ of $P$ (if $\y$ continues it must be at a child of $P$ by the second item). We argue now that $\y$ must in fact stop at $P'$, which concludes the proof. First note that since $\x$ stops at $P$, it means that line \ref{line:newskel} occured for some child $Q_j$ of $Q \preceq P$ obtained by splitting $B(Q)=\{a_1,\hdots,a_{\ell}\}$ into $\{a_1,\hdots,a_j\} > \{a_{j+1},\hdots, a_{\ell}\}$. In other words one has:
\[\left|\gap_{Q_j}(\x)-c(Q)\cdot \range_Q(\x)\right| \leq (d_{\mathcal T_{K,m}}(P,Q) +1) \cdot 6\varepsilon .\]
Now observe that, since $|\x-\y|_{\ell^{\infty}} \leq \varepsilon$, one has $|\range_Q(\x)-\range_Q(\y)| \leq 2 \varepsilon$. Next denote $b_1,\hdots,b_{\ell}$ a permutation of $a_1,\hdots, a_{\ell}$ such that $y(b_1) \geq \hdots \geq y(b_{\ell})$ and let $j'$ be the index $j$ of the largest $y(j)$ which is $\varepsilon$-close to $x(a_j)$. Now consider the child $Q'$ of $Q$ obtained by splitting $B(Q)$ into $\{b_1, \dots, b_{j'}\}>\{b_{j'+1},\dots,b_{\ell}\}$. One has $|\gap_{Q'}(\y)-\gap_{Q_j}(\x)| \leq 2 \varepsilon$. In particular we obtain by the triangle inequality:
\[\left|\gap_{Q'}(\y)-c(Q)\cdot \range_{Q}(\y)\right| \leq 4 \varepsilon +(d_{\mathcal T_{K,m}}(P,Q) +1) \cdot 6\varepsilon = (d_{\mathcal T_{K,m}}(P',Q) +1) \cdot 6\varepsilon .\]
In particular, in the run on $\y$, when the while loop is at $P'$, line \ref{line:newskel} will occur when invoked on $Q'$, which shows that the run stops at $P'$.
\end{proof}

The next lemma establishes three basic guarantees for Algorithm~\ref{alg:partition}. In particular, it ensures that the DOP associated to $\x$ describes the order of the coordinates of $\x$, e.g. if $\mathcal{P}_{c,\varepsilon}(\x)$ is of the form $\left[ \{1\}>\{2\}>\dots>\{K\} \right]$, then the coordinates of $\x$ are in nonincreasing order.

\begin{lem}\label{lem:guarantee}
Suppose Algorithm~\ref{alg:partition} is run on $\x \in [0,1]^K$. Suppose that the operation $P\leftarrow \widehat P$ occurs during Algorithm~\ref{alg:partition} (in Line~\ref{line:child}). Then:

\begin{enumerate}

	\item \label{it:guar1} $\x$ obeys all the inequalities of $\widehat P$.

	\item \label{it:guar2} If for any ancestor $Q\preceq P$ and any child $Q_j$ of $Q$ the inequality
	\begin{equation}
	\label{eq:10K}|\gap_{Q_j}(\x) -  c(Q) \cdot \range_{Q}(\x)| \leq 10K \varepsilon.
	\end{equation}
never holds, and if $\y\in [0,1]^K$ satisfies $|\x-\y|_{\ell^{\infty}}\leq \varepsilon$, then $P\leftarrow \widehat P$ also occurs when Algorithm~\ref{alg:partition} is run on $\y$.
	\item \label{it:guar3} If $\widehat P\notin \L(\mathcal T_{K,m})$ is not a leaf and is the final output of Algorithm~\ref{alg:partition}, then there exists some ancestor $Q\preceq \hat P$ with a child $Q_j\in \CHILDREN(Q)$ such that inequality~\eqref{eq:10K} holds.
\end{enumerate}
\end{lem}

\begin{proof}
To show the first assertion, we simply observe that Algorithm~\ref{alg:partition} sorts the actions of $B(P)$ in line~\ref{writeB(Q)} according to their values at the point $\x$, hence every inequality sign added is true for $\x$.

For the second assertion, we observe that the factor $10K$ in Equation~\ref{eq:10K} is so large that, by the triangle inequality, Line~\ref{line:newskel} of the algorithm applies to neither $\x$ nor $\y$ for any $Q_j$ a child of $Q\preceq P$. Moreover for these $Q_j$ and $Q$, the signs of $\left(\gap_{Q_j}(\x) -  c(Q) \cdot \range_{Q}(\x)\right)$ and $\left(\gap_{Q_j}(\y) -  c(Q) \cdot \range_{Q}(\y)\right)$ always agree, so $x,y$ behave identically in Line~\ref{line:childineq}. Since $x,y$ do not terminate before $\widehat P$ in Line~\ref{line:newskel} and behave identically in Line~\ref{line:childineq} before $P\leftarrow \widehat P$, we conclude that $\widehat P\leftarrow P$ occurs when Algorithm~\ref{alg:partition} is run on $\y$.

Finally for the last assertion, because $\widehat P$ is not a leaf, Line~\ref{line:newskel} of Algorithm~\ref{alg:partition} must have been the reason to terminate the algorithm on $\x$. As $\mathcal T_{K,m}$ has depth at most $K$, inequality~\ref{eq:10K} must hold just before termination.
\end{proof}

\subsection{Coloring the partition}\label{subsec:coloring}

We now want to turn our partition into a full strategy. For this, we need a rule specifying, for each DOP $P$ and each player $X$, which arm player $X$ should play in the partition element $P$. We construct here a ``robust" rule, such that when eventually combined with Lemma~\ref{lem:topology} it will give a collision-free strategy for the players. We start with two definitions.

\begin{defn}
For a DOP $P$, define $\Feas_P\subseteq \binom{[K]}{m}$ to consist of all $m$-subsets of $[K]$ which comprise the top $m$ actions in some total ordering extending $P$.
\end{defn}
Note that this is more stringent than only requiring that each element might individually be in the top $m$. In particular, sequences in $\Feas_P$ contain all elements of $A(P)$ and a fixed size subset of $B(P)$.

\begin{defn}
An $m$-coloring of $\mathcal T_{K,m}$ is a function $F:\mathcal T_{K,m}\to [K]^m$. An $m$-coloring $F=(f_1, \hdots, f_m)$ is called \emph{collision-robust} if for any $P,Q\in\mathcal T_{K,m}$ with $d_{\mathcal T_{K,m}}(P,Q) \leq 1$ and any $i,j\in [m]$ with $f_i(P)=f_j(Q)$, one must necessarily have $i=j$. 
\end{defn}

Now, fix $c,\varepsilon$ and suppose we are given an arbitrary function $G:\mathcal T_{K,m}\to\binom{[K]}{m}$ such that $G(P)\in \Feas_P$ for all $P$. Then we claim that there is a robust $m$-coloring $F$ such that $F(P)$ is a permutation of $G(P)$ for all $P$. To see this, we simply proceed recursively down the tree. We choose $F(\mathrm{ROOT})$ to be the lexicographically first (i.e. sorted) permutation of $G(\mathrm{ROOT})$, and then for $P\in\mathcal T_{K,m}$ choose $F(P)$ to be a permutation of $G(P)$ with maximum possible overlap with $F(\PARENT(P))$ (if there are several such permutations, pick the first lexicographically). That is, if $f_i(\PARENT(P)) \in G(P) \cap G(\PARENT(P))$, we have $f_i(P)=f_i(\PARENT(P))$. It is easy to see that this produces a robust $m$-coloring. Hence, we have shown the following.

\begin{lem}
\label{lem:color}
For any \[G:\mathcal T_{K,m}\to\binom{[K]}{m}\] such that $G(P)\in \Feas_P$ for all $P$, there is a collision-robust coloring \[F:\mathcal T_{K,m}\to [K]^m,\] where $F(P)$ is always a permutation of $G(P)$.
\end{lem}

\section{The full feedback scenario} \label{sec:fullinfo}
We consider here the {\em full information} version of the problem described in the introduction, where the players observe at each round a reward on {\em all} the arms (not only the one played as in the bandit case), but these rewards are independent realizations for each player. We propose to use the construction from the previous section to build a strategy for this full information game. We fix a robust $m$-coloring $F$ of $\mathcal T_{K,m}$ given by applying Lemma~\ref{lem:color} to the function $G$ which sets $G(P)$ to be the lexicographically first element of $\Feas_P$. We take $c:\mathcal T_{K,m}\to\left[0,\frac{1}{K}\right]$ to be an i.i.d. uniform function of the distance to the root in $\mathcal{T}_{K,m}$. That is, we set
\[c(P)=C(d(P,\ROOT)) ,\]
where $C(0),C(1),\dots,C(K-1)$ are i.i.d. uniform variables on $\left[0,\frac{1}{K}\right]$. For each time $t \in [T]$, we write
\[\varepsilon_t := 10 \sqrt{\frac{\log (mKT)}{t}}.\]
The strategy followed by player $X$ at time $t$ is then the following:
\begin{enumerate}
	\item Denote by $\q_t^X = \left( q_t^X(i) \right) \in [0,1]^K$ the vector of empirical mean rewards of the arms observed by $X$ from time $1$ to time $t-1$.
	\item Apply Algorithm~\ref{alg:partition} to find $\mathcal P_{c, \varepsilon_t} (\q_t^X) \in\mathcal T_{K,m}$.
	\item Play the action $i_t^X = f_X(\mathcal P_{c, \varepsilon_t} (\q_t^X))$, where $F=(f_1,\dots,f_m)$.
\end{enumerate}

The key observation is the following lemma that will be used to bound the probability that Algorithm~\ref{alg:partition} stops at an inner node. This is important to estimate, since this event may result in suboptimal arm choices.
\begin{lem}
\label{lem:cost}
Let $\delta>0$ and $\x \in [0,1]^K$. For any $P\in\mathcal T_{K,m}$ of depth $h$ and any child $Q$ of $P$, we have
\[\mathbb P\left( \left| \gap_{Q}(\x)-C(h)\cdot \range_{P}(\x)\right|\leq \delta ~|~ C(0), \hdots, C(h-1) \right) \leq \frac{2K\delta}{\range_{P}(\x)}.\]
\end{lem}

\begin{proof}
By construction $C(h)$ is uniform in $\left[0,\frac{1}{K}\right]$ even under the conditioning. Moreover, the set of values $C(h)$ for which the event on the left occurs is an interval of length $\frac{2\delta}{\range_{P}(\x)}$, which implies the claim.
\end{proof}

Our main result for this section is as follows:
\begin{thm} 
\label{thm:goodpartition} 
The strategy described above satisfies for any $\p \in [0,1]^K$:
\[
\mathbb{E}[R_T] \leq O(m K^4 \sqrt{T \log(T)} ).
\]
Furthermore with probability at least $1-\frac{1}{T}$ the players never collide.\footnote{$m$ can be replaced easily by $\min(m,K-m)$. The same holds in Theorem~\ref{thm:bandit}.}
\end{thm}

\begin{proof}
Let us first prove the non-collision property. By the Hoeffding inequality and a crude union bound over times, players and arms, with probability at least $1-\frac{1}{T}$, it holds that for every $t$ and each player $X$, we have
\begin{equation}\label{eqn_error_bound_fullinfo}
\left| \q_t^X-\p \right|_{\ell^{\infty}} \leq \frac{1}{3} \varepsilon_t.
\end{equation}
Moreover, if this occurs, for any two players $X$ and $Y$, we have $|\q_t^X-\q_t^Y| < \varepsilon_t$, so by Lemma~\ref{lem:topology} the DOPs $\mathcal{P}_{c,\varepsilon_t} \left( \q_t^X \right)$ and $\mathcal{P}_{c,\varepsilon_t} \left( \q_t^Y \right)$ are neighbour vertices in $\mathcal T_{K,m}$. By our choice of $F$ using Lemma~\ref{lem:color}, this ensures that there is no collision, i.e. that $X$ and $Y$ do not play the same arm.
\newline

From now on, all the information that we will use in the proof will be the description of the mappings $\mathcal{P}_{c,\varepsilon_t}$ together with the ``stability" result of~\eqref{eqn_error_bound_fullinfo}. We highlight right now that the exact same argument will be needed in the end of the proof in the bandit setting (Section~\ref{sec:bandit}).

We now control the regret under the event \eqref{eqn_error_bound_fullinfo} for a fixed time step $t$. For $\x \in [0,1]^K$, we denote by $P^0(\x)=\ROOT$, $P^1(\x)$, $P^2(\x), \hdots,$ the random (because of $c$) path in $\mathcal T_{K,m}$ visited during the while loop of Algorithm \ref{alg:partition} (line \ref{while}) when we run it on $\x$ with parameters $c$ and $\varepsilon_t$. If the output of the algorithm is at depth $k$ in $\mathcal{T}_{K,m}$, then we denote $P^h(\x) = P^k(\x)$ for any $h \geq k$. Note that $P^1(\x),\hdots,P^h(\x)$ only depends on $C(0), \hdots, C(h-1)$. We also write $P_j^h(\x)$ for the child of $P^h(\x)$ obtained by splitting $B(P^h(\x))=\{a_1,\hdots,a_{\ell}\}$, where $x(a_1)\geq x(a_2)\geq\dots\geq x(a_{\ell})$, into $\{a_1,\dots,a_j\}>\{a_{j+1},\dots,a_{\ell}\}$. For $h \geq 0$ and $j \geq 1$, we consider the events:
\begin{eqnarray*}
E_{h,j} & = & \bigg\{ \left|\gap_{P_j^h(\p)}(\p) -  C(h) \cdot \range_{P^h(\p)}(\p) \right| \leq 10 K \varepsilon_t \\
& & \text{ and } \left| \gap_{P_i^k(\p)}(\p) -  C(k) \cdot \range_{P^k(\p)}(\p) \right| > 10 K \varepsilon_t \text{ for any } k <h \text{ and any } i \bigg\} .
\end{eqnarray*}
The first key observation is that on the event $E_{h,j}$, it must be that $P^h(\p)$ is a node of depth $h$, and furthermore for any $\x$ such that $|\x - \p|_{\ell^{\infty}} \leq \varepsilon_t$ one must also have $P^h(\x) = P^h(\p)$. These follow from items~\ref{it:guar3},~\ref{it:guar2} of Lemma~\ref{lem:guarantee} - part ~\ref{it:guar3} shows $P^h(\p)$ has depth $h$ and part~\ref{it:guar2} shows $P^h(\x)$ takes the same value. Recalling item~\ref{it:guar1} of the same lemma, on the event $E_{h,j}$, we know that the total regret of the $m$ players is upper bounded by $m\cdot \range_{P^h(\p)}(\p)$. Indeed, the top $m$ actions of $\p$ can only differ from the players' actions in the choice of subset of $B(P^h(\x))=B(P^h(\p))$. The second key observation is that by Lemma \ref{lem:cost} one has
\[
\mathbb P(E_{h,j} | C(0), \hdots, C(h-1)) \leq \frac{20 K^2 \varepsilon_t}{\range_{P^h(\p)}(\p)} \,.
\]
Thus we obtain that the contribution to the regret from $E_{h,j}$ is bounded by $20 m K^2 \varepsilon_t$. If none of the events $E_{h,j}$ occurs, then the players play exactly the top $m$ actions according to $\p$, hence incur zero regret. Indeed, this follows from the same discussion as above using Lemma~\ref{lem:guarantee}, but if $P^h(\p)$ is a leaf of $\mathcal{T}_{K,m}$, then $\range_{P^h(\p)}(\p)=0$. Finally, summing over $h$, $j$ and $t$ yields the regret estimate
\begin{equation}\label{eqn:regret_fullinfo_with_epsilon}
\mathbb{E}[R_T] \leq O \left( mK^4 \sum_{t=1}^T \varepsilon_t \right).
\end{equation}
The theorem follows.
\end{proof}

\section{The bandit scenario} \label{sec:bandit}
For the bandit version, we again use mappings of the form $\mathcal P_{c,\varepsilon_t}$, with the exact same $c$ as in the full information case but a larger $\varepsilon_t$. The important difference will be that, in order to avoid neglecting the exploration of some of the arms, we will use a different, randomized coloring of $\mathcal{T}_{K,m}$.
Specifically, at each time $t$ we apply an uniformly random permutation $\pi_t:[K]\to [K]$ to the actions in defining the lexicographic ordering used in Section~\ref{subsec:coloring}, where the $\pi_t$ are independent. This defines a $\pi_t$-random coloring of the vertices of $\mathcal{T}_{K,m}$ and preserves the collision-robustness of Lemma~\ref{lem:color}. Moreover, by symmetry, the randomness of $\pi_t$ causes each $F(P)$ to contain a uniformly random subset of $B(P)$ of the appropriate size $m-|A(P)|$, and in particular to contain any arm $i\in B(P)$ with probability at least $\frac{1}{K}$.














We can now describe the strategy. We let $n^X_t(i)$ be the number of times player $X\in [m]$ sampled arm $i\in [K]$ in the first $t-1$ time steps, and let $r^X_t(i)\leq n^X_t(i)$ be the amount of reward observed so far. We let $q^X_t(i)$ be the empirical estimate of $p(i)$ by player $X$ at the start of time $t$, defined by 
\[q^X_t(i)=\frac{r^X_t(i)}{n^X_t(i)}\in [0,1].\]
For the first $T_0=10^9K\log(KT)$ time-steps, we simply have player $X$ play arm $X+t\pmod{K}$ at time $t$. After that, at time $t>T_0$ the players as before play via the mapping $\mathcal P_{c,\varepsilon_t}$, i.e. player $X$ plays arm $f_X \left( \mathcal{P}_{c, \varepsilon_t} \left( \mathbf{q}^X_t\right) \right)$, where $F=\left( f_1, \dots, f_m \right)$ is our $\pi_t$-random coloring, $c$ is a uniform function of the distance to the root as in Section~\ref{sec:fullinfo}, and
\[\varepsilon_t =10000\sqrt{\frac{K^3\log(KT)}{t}}.\]

We now begin the analysis. We define the events:
\begin{align*}
\Omega_1&=\left\{ \forall t\geq T_0,i\in [K], X\in [m], \text{ we have } q_t^X(i)-p(i)|<\frac{\varepsilon_{n_t^X(i)}}{100K^{3/2}} \right\},\\
\Omega_2&=\Big\{ \forall t\geq T_0, X\in [m], \forall P \in \mathcal{T}_{K,m}, \text{ if }\forall s\leq t \text{ we have } \mathcal{P}_{c,\varepsilon_s}(\q_s^X)\preceq P, \\
&\quad\quad\quad\text{ then } \forall i\in A(P)\cup B(P), \, n_t^X(i)\geq \left\lfloor \frac{t}{2K} \right\rfloor \Big\},\\
\Omega&=\{\Omega_1\text{ and } \Omega_2\}.
\end{align*}

We observe that \[\frac{\varepsilon_{t/(2K)}}{100K^{3/2}}\leq \frac{\varepsilon_t}{10K}.\] This means that when the conclusions of $\Omega$ hold, we have
\[|q_t^X(i)-p(i)|<\frac{\varepsilon_t}{10K}.\]

\begin{lem}
Using the above strategy, for any choice of $\left( C(h) \right)$ and any $\p\in [0,1]^K$, we have:
\[\mathbb P[\Omega]\geq 1-\frac{1}{T}.\]
\end{lem}

\begin{proof}



We show that each of $\Omega_1,\Omega_2$ have probability at least $1-\frac{1}{2T}$. For $\Omega_1$ this follows immediately from Hoeffding's inequality.

We now show $\mathbb P[\Omega_2]\geq 1-\frac{1}{2T}$. This is where we will use the randomization of the coloring using the $\pi_t$ to explore evenly. For $t$ small, we use the initial sampling phase (this is the only place we use the initial phase). Indeed, for $t\leq \frac{3T_0}{2}$ the inequality $n_t^X(i)\geq \frac{t}{2K}$ is immediate given our initial $T_0$ rounds of perfectly uniform sampling. Now, fix $t\geq \frac{3T_0}{2}$, $X\in [m]$, and $i\in [K]$. Let $E_{t,X}(i)$ denote the event that there is $P$ for which $i \in A(P) \cup B(P)$ and, for all $s \leq t$, we have $\mathcal{P}_{c,\varepsilon_t}(\q_s^X) \preceq P$. Since we use a uniform random permutation $\pi_t$ at each time $s \leq t$, conditionally on everything that happened before, the probability for $X$ to play $i$ is at least $\frac{1}{K}$. It follows that
\[\mathbb P\left( E_{t,X}(i) \text{ but } n_t^X(i)<\frac{t}{2K}\right)\leq \mathbb P\left( Bin\left(t-T_0,\frac{1}{K}\right)\leq \frac{t-T_0}{2K}\right).\]
As $t-T_0\geq 1000K\log(KT)$, the right hand probability is at most $\frac{1}{2mKT^2}$ by applying the multiplicative Chernoff estimate $\mathbb P \left( Bin(N,p)\leq \frac{Np}{2} \right)	\leq e^{-\frac{Np}{8}}$ in \cite[Theorem 4.5]{mitzenmacher2017probability}. Union bounding over all $t,X,i$ concludes the proof.
\end{proof}

From now on, our reasoning will be deterministic. As time increases and we get further in the tree $\mathcal{T}_{K,m}$, we need to have a reasonable estimate of $\p$ to know which arms to keep exploring, but we also need to explore the right arms to have a good estimate (and avoid collisions). Therefore, we will use a ``bootstrap" argument as in~\cite{BB20}, proving a certain property $\Gamma_t$ by induction on $t$.

To define $\Gamma_t$, we first need a few definitions. For $k\geq 0$, let $\tau_k$ be the first time $t>T_0$ that $\mathcal P_{c,\varepsilon_t}(\q_t^X)\succeq P^k$ holds for some $P^k$ of depth $k=d(\ROOT,P^k)$ and some $X^k\in [m]$ (in particular $\tau_0=T_0+1$). Throughout this section, we will keep denoting by $P^k$ the corresponding DOP, and by $X^k\in [m]$ the corresponding player (choosing an arbitrary pair $(P^k,X^k)$ if there are multiple instances at time $\tau_k$). We let $k_t=\max\{k:\tau_k<t\}$. We now introduce the event $\Gamma_t$ for each $t> T_0$.

\begin{defn}
For $t>T_0$, we denote by $\Gamma_t$ the event
\[\Gamma_t=\{\Gamma_t^1\text{ and }\Gamma_t^2\text{ and }\Gamma_t^3\},\] where $\Gamma_t^1,\Gamma_t^2,\Gamma_t^3$ are the three events:
\begin{enumerate}
	\item \label{it:nocoll}For all $r\leq t$ and $X,Y\in [m]$ we have
\[d_{\mathcal T_{K,m}}(\mathcal P_{c,\varepsilon_r}(\q_r^X), \mathcal P_{c,\varepsilon_r}(\q_r^Y))\leq 1.\]
In particular, there is no collision in the first $t$ time-steps.
	\item \label{it:nosib} For any $k,s \geq 0$ satisfying $\tau_k\leq \min(s,t)$ and any $Q^k\neq P^k$ a sibling of $P^k$, we have
\[\mathcal P_{c,\varepsilon_s}(\q^Y_s)\not\succeq Q^k.\]
\item \label{it:accurate} For any $k \geq 1$ and $s \geq 0$ satisfying $\tau_k\leq \min(s,t)$ and for any $i\in A(P^{k-1})\cup B(P^{k-1})$, we have
\[|q_s^Y(i)-p(i)|\leq \frac{\varepsilon_{\tau_k}}{10K}.\]
\end{enumerate}
\end{defn}

Our goal is to show that $\Omega$ implies $\Gamma_T$ (which of course subsumes $\Gamma_t$ for all $t\leq T$). We will prove that $\Gamma_t$ holds inductively in $t$ below in Proposition~\ref{lem:banditmain}. We first state a simple consequence of $\Gamma_t$ which is useful in proving Proposition~\ref{lem:banditmain}.

\begin{lem} \label{prop:gammaprep}
Let $t>T_0$, and suppose that $\Gamma_{t-1}$ holds. Then:
\begin{enumerate}
\item\label{it:chain} \[\ROOT=P^0\prec P^1\prec\dots\prec P^{k_t}.\]
\item\label{it:hist} For any $r\leq t, X\in [m]$, there is $k_r^X \leq k_{t+1}$ such that $\mathcal P_{c,\varepsilon_r}(\q_r^X)=P^{k_r^X}$.
\item\label{it:future} If $\tau_{k_{t+1}}=t$ holds, then $P^{k_{t+1}}\in \CHILDREN(P^{k_t})$.
\end{enumerate}
\end{lem}

\begin{proof}
All three items follow immediately from $\Gamma^2_{t-1}$.
\end{proof}

\begin{prop}\label{lem:banditmain}
If $\Omega$ holds, then $\Gamma_T$ also holds.
\end{prop}

\begin{proof}
We prove $\Gamma_t$ holds for all $t\geq T_0$ by induction on $t$. First note that $\Gamma^1_{T_0}$ and $\Gamma^2_{T_0}$ follow from the definition of the initial exploration phase, and $\Gamma^3_{T_0}$ is an empty statement, since $\tau_1>T_0$ by definition.

Hence, we now assume $\Omega$ and $\Gamma_{t-1}$ and prove $\Gamma_t^1,\Gamma_t^2,\Gamma_t^3$.

\paragraph{Proof of $\Gamma_t^3$.}
We fix $k \geq 1$ and $s \geq 0$ satisfying $\tau_k\leq \min(s,t)$. Since the definition of $\Gamma_{t-1}^3$ does \emph{not} require that $s\leq t-1$, there is nothing to prove unless $\tau_{k_{t+1}}=t$. We therefore assume $\tau_{k_{t+1}}=t$. By Lemma~\ref{prop:gammaprep}, we have
\[\ROOT=P^0\prec P^1\prec\dots\prec P^{k_{t+1}}.\]
Moreover, for any any $Y\in [m]$, there is $k_t^Y\leq k_{t+1}$ such that $\mathcal{P}_{c,\varepsilon_t}(\q^Y_t)=P^{k_t^Y}$. Using the event $\Omega_2$, we obtain for any $i\in A(P^{k-1})\cup B(P^{k-1})$ the assertion of $\Gamma_t^3$, namely:
\[ \left| q_s^Y(i)-p(i) \right| \leq\frac{\varepsilon_{n_s^X(i)}}{100K^{3/2}}\leq \frac{\varepsilon_{\tau_k}}{10K}.\]

\paragraph{Proof of $\Gamma_t^1$.}
Since $\Gamma^1_{t-1}$ holds, we only need to prove the statement for $r=t$. Invoking Lemma~\ref{prop:gammaprep}, we have: \[\ROOT=P^0\prec P^1\prec\dots\prec P^{k_t}.\] Hence for each $k\leq k_t$, we have by the definition of Algorithm~\ref{alg:partition} (and more precisely the fact that we have not returned $P^{k-1}$ according to line \ref{line:newskel}):
\begin{equation}\label{eq:biggap}
\left|\gap_{P^k} \left( \q^{X^k}_{\tau_k} \right)-c(P^{k-1})\cdot \range_{P^{k-1}} \left( \q^{X^k}_{\tau_k} \right) \right|\geq 4\varepsilon_{\tau_k}.
\end{equation}
Applying $\Gamma_t^3$ (proved just above, and whose hypothesis holds for $s\leq t$ as we assume $\Gamma_{t-1}^1)$ implies that for $t\geq s\geq \tau_k$, $i\in A(P^{k-1})\cup B(P^{k-1})$ and any $Y\in [m]$ we have 
\begin{equation}\label{eq:accurate}
\left| q_s^Y(i)-p(i) \right| \leq \frac{\varepsilon_{\tau_k}}{10K}.
\end{equation}
Combining Equation~\eqref{eq:biggap} with Equation~\eqref{eq:accurate} for $(t,Y)$ and $(\tau_k,X^k)$ and using the triangle inequality, we deduce:
\begin{equation}\label{eq:stillgap}\left|\gap_{P^k}(\q^{Y}_{t})-c(P^{k-1})\cdot \range_{P^{k-1}}(\q^{Y}_{t})\right|\geq 2\varepsilon_{\tau_k}.\end{equation}
Applying Equation~\eqref{eq:accurate} for $(t,Y)$ and $(t,Z)$ with $Y,Z\in [m]$ and the same $k$, we obtain:
\begin{equation}\label{eq:samegap}
\left|\left(\gap_{P^k}(\q^{Y}_{t})-c(P^{k-1})\cdot \range_{P^{k-1}}(\q^{Y}_{t})\right)-\left(\gap_{P^k}(\q^{Z}_{t})-c(P^{k-1})\cdot \range_{P^{k-1}}(\q^{Z}_{t})\right)\right|\leq\frac{\varepsilon_{\tau_k}}{K}.
\end{equation}
Finally, combining Equations~\eqref{eq:stillgap} and~\eqref{eq:samegap}, we conclude:
\begin{equation}\label{eq:relerror}
\frac{\left|\left(\gap_{P^k}(\q^{Y}_{t})-c(P^{k-1})\cdot \range_{P^{k-1}}(\q^{Y}_{t})\right)-\left(\gap_{P^{k-1}}(\q^{Z}_{t})-c(P^{k-1})\cdot \range_{P^{k-1}}(\q^{Z}_{t})\right)\right|}{\left|\gap_{P^k}(\q^{Y}_{t})-c(P^{k-1})\cdot \range_{P^{k-1}}(\q^{Y}_{t})\right|}\leq \frac{1}{2K}.
\end{equation}

Therefore, for any $k,t$ with $\tau_k \leq t$ and $Y,Z\in [m]$, if 
\[\left|\gap_{P^k}(\q^{Y}_{s})-c(P^{k-1})\cdot \range_{P^{k-1}}(\q^{Y}_{s})\right|\leq 4d\varepsilon_s\]
holds for some $d\leq K+1$, then it follows that:
\[\left|\gap_{P^k}(\q^{Z}_{s})-c(P^{k-1})\cdot \range_{P^{k-1}}(\q^{Z}_{s})\right|\leq 4(d+1)\varepsilon_s.\]

We recall that $\mathcal P_{c,\varepsilon_t}(\q^Y_t)=P^{k_t^Y}$ for some $k_t^Y \leq k_{t+1}$. Using the description of Algorithm~\ref{alg:partition}, the last equations mean that if $\mathcal{P}_{c,\varepsilon_t}(\q_t^Y)=P^{k_t^Y}$ is decided because the condition of line~\ref{line:newskel} is fulfilled for $P=P^{k_t^Y}$ and $Q$, then the same condition will be fulfilled for $P=P^{k_t^Y+1}$ and the same $Q$ when we run the algorithm for $\q_t^Z$. Since $\mathcal P_{c,\varepsilon_t}(\q^Z_t)=P^{k_t^Z}$, this proves $k_t^Z \leq k_t^Y+1$. The reasoning is symmetric in $Y$ and $Z$, so we obtain $d_{\mathcal T_{K,m}}(\mathcal P_{c,\varepsilon_t}(\q^Z_t),\mathcal P_{c,\varepsilon_t}(\q^Y_t))\leq 1$. This concludes the inductive step for part~\ref{it:nocoll}.

\paragraph{Proof of $\Gamma_t^2$.}
We fix $(k,s,t)$ satisfying $\tau_k\leq \min(t,s)$. We apply the second item of Lemma~\ref{lem:topology} with $\x=\q^{X^k}_{\tau_k}$, $\y=\q^Y_s$, $\varepsilon=\varepsilon_{\tau_k}$, $\varepsilon'=\varepsilon_s$ and $P=P^{k-1}$. Note that the assumption $\left| q^{X^k}_{\tau_k}(i)-q^Y_s(i) \right| \leq \varepsilon_{\tau_k}$ for $i \in A(P^{k-1}) \cup B(P^{k-1})$ is satisfied by the event $\Gamma^3_t$ proved above. We obtain that $\mathcal{P}_{c,\varepsilon_{\tau_k}} \left( \q^{X^k}_{\tau_k} \right)=P^k$ and $\mathcal{P}_{c,\varepsilon_s} \left( \q_s^Y \right)$ cannot be descendants of two distinct children of $P^{k-1}$, which shows that $\Gamma^2_t$ holds. This completes our induction over $t$ and concludes the proof of Proposition~\ref{lem:banditmain}.

\end{proof}

Having established the fundamental properties on the behavior of our algorithm, we now turn to the regret analysis. Unlike the full-feedback scenario, the estimation error of a coordinate at time $t$ depends not on $\varepsilon_t$ but the potentially much larger $\varepsilon_{n_t^X(i)}$. We circumvent this issue in the next lemma by showing that the high-error coordinates can be set to their exactly correct values without affecting the result of the space partition.

\begin{lem}\label{lem:modify}
Let $X\in [m]$ be an arbitrary player. Let $t\geq T_0$, and let $k_*=k_{t/(4K^2)}=\max\{k:\tau_{k}<t/(4K^2)\}$. Define the modified probability estimate $\widetilde \q_t^X$ by:
\[\widetilde q_t^X(i)=\begin{cases} q_t^X(i), & i\in A(P^{k_*})\cup B(P^{k_*}),\\ p(i),& else.  \end{cases}\]
Then assuming $\Omega$ holds, we have 
\[\mathcal P_{c,\varepsilon_t}(\q_t^X)=\mathcal P_{c,\varepsilon_t}(\widetilde \q_t^X).\] 
\end{lem}

\begin{proof}
Roughly speaking, the reason why this is true is that at each time-step, the arms that have not been explored a lot are not relevant to make further decisions. More precisely, let $k\leq k_*$. The idea will be to prove that
\begin{align*}\left|\gap_{P^k}(\q_t^X)-c(P^{k-1})\cdot \range_{P^{k-1}}(\q_t^X)\right|&\geq 6K\varepsilon_t,\\
 \left|\gap_{P^k}(\widetilde \q_t^X)-c(P^{k-1})\cdot \range_{P^{k-1}}(\widetilde \q_t^X)\right|&\geq 6K\varepsilon_t.
\end{align*}
Once we show this, it will follow that $\q_t^X,\widetilde \q_t^X$ behave identically at all stages of Algorithm~\ref{alg:partition} which involve a coordinate they differ on. The idea is simply that $\varepsilon_t$ is much smaller than $\varepsilon_{t/(4K^2)}$, so the interfaces shrank a lot between times $t/(4K^2)$ and $t$. Hence, if we were outside the interfaces at time $\tau_k<\frac{t}{4K^2}$, we are definitely still outside at time $t$. Indeed, by the definition of $\tau_k$ and Algorithm~\ref{alg:partition}, we have for some $Y\in [m]$:
\[\left|\gap_{P^k}(\q_{\tau_k}^Y)-c(P^{k-1})\cdot \range_{P^{k-1}}(\q_{\tau_k}^Y)\right|\geq 4\varepsilon_{\tau_k}.\] 
From here, applying $\Gamma_T^3$ and the triangle inequality shows that 
\begin{align*}\left|\gap_{P^k}(\q_{t}^X)-c(P^{k-1})\cdot \range_{P^{k-1}}(\q_{t}^X)\right|&> 3\varepsilon_{\tau_k} \geq 6K\varepsilon_t,\\
\left|\gap_{P^k}(\widetilde \q_{t}^X)-c(P^{k-1})\cdot \range_{P^{k-1}}(\widetilde \q_{t}^X)\right|&> 3\varepsilon_{\tau_k} \geq 6K \varepsilon_t.
\end{align*}
Note that the reason why the second inequality holds is that by definition, the quantities $\gap_{P^k}(\x)$ and $\range_{P^{k-1}}(\x)$ only depend on the coordinates $x(i)$ such that $i \in A(P^{k-1}) \cup B(P^{k-1})$. These are the coordinates for which $\Gamma^3_T$ provides an estimate.

Finally, we know from Propopsition~\ref{lem:banditmain} that $\mathcal P_{c,\varepsilon_t}(\q_t^X)=P^{k_t^X}$ for some $k_t^X$, and since the equations just above hold for all $k\leq k_*$ we conclude that $\mathcal P_{c,\varepsilon}(\q_t^X),\mathcal P_{c,\varepsilon}(\widetilde \q_t^X)\succeq P^{k_*}$ and moreover that line~\ref{line:newskel} of Algorithm~\ref{alg:partition} never comes into effect for $Q=P^k$ with $k<k_*$. We finally observe that $\q_t^X,\widetilde \q_t^X$ now exactly agree in all still-relevant coordinates $i\in A(P^{k_*})\cup B(P^{k_*})$ and hence end up in the same region at the end of Algorithm~\ref{alg:partition}.
\end{proof}

Lemma~\ref{lem:modify} will now imply the final regret bound.

\begin{thm}
The expected regret in the bandit case is 
\[O\left(mK^{11/2}\sqrt{T\log(T)}\right).\]
\end{thm}

\begin{proof}
We assume $\Omega$ holds throughout. We then have, for all $i$:
\begin{equation}\label{eqn:qtilde_and_q}
\left| \widetilde q_t^X(i)-p_t(i) \right| \leq\frac{\varepsilon_{n_{t/(4K^2)}^X(i)}}{100K^{3/2}}\leq  \frac{\varepsilon_{t/(10K^3)}}{10K^{3/2}}< \frac{\varepsilon_t}{3}.
\end{equation}

By Lemma~\ref{lem:modify} we have $\mathcal P_{c,\varepsilon_t}(\q_t^X)=\mathcal P_{c,\varepsilon_t}(\widetilde \q_t^X)$. Therefore the regret of the strategy defined by $(\q_t^X)_{t\in [T],X\in [m]}$ is equal to the regret of the ``cheating strategy" obtained by playing using the estimates $\widetilde \q_t^X$ instead of $\q_t^X$. Now note that by Equation~\eqref{eqn:qtilde_and_q}, the ``cheating strategy" using $\widetilde \q_t^X$ satisfies Equation~\eqref{eqn_error_bound_fullinfo}. Therefore, as argued in the proof of Theorem~\ref{thm:goodpartition}, we obtained the regret bound
\[\mathbb{E}[R_T] \leq O \left( mK^4 \sum_{t=1}^T \varepsilon_t \right).\]
The theorem follows by using the definition of $\varepsilon_t$.
\end{proof}

\bibliographystyle{alpha}
\bibliography{newbib}
\end{document}